\documentclass{article}
\usepackage{ijcai19}
\pdfoutput=1
\usepackage{times}
\usepackage{soul}
\usepackage{url}
\usepackage[hidelinks]{hyperref}
\usepackage[utf8]{inputenc}
\usepackage[small]{caption}
\usepackage{graphicx}
\usepackage{amsmath}
\usepackage{booktabs}
\usepackage{algorithm}
\usepackage{algorithmic}
\usepackage{stfloats}
\usepackage{amssymb,amsmath,amsthm, amsfonts}
\usepackage{mathtools}
\newtheorem{theorem}{Proposition}[section]
\newtheorem{corollary}{Corollary}[theorem]
\newtheorem{lemma}[theorem]{Lemma}
\usepackage{subcaption}
\usepackage{mathtools}
\usepackage{commath}

\usepackage{color, colortbl}
\definecolor{LightCyan}{rgb}{0.88,1,1}
\definecolor{Gray}{gray}{0.9}

\DeclareRobustCommand{\VAN}[3]{#2}

\title{Autoregressive Policies for Continuous Control Deep Reinforcement Learning}

\author{
  Dmytro Korenkevych$^1$
 \and
A. Rupam Mahmood$^1$
 \and
 Gautham Vasan$^1$
 \and 
 James Bergstra$^1$
 \affiliations
$^1$Kindred AI
\emails
\{dmytro.korenkevych, rupam, gautham.vasan, james\}@kindred.ai 
}
\begin{document}

\maketitle

\begin{abstract}
Reinforcement learning algorithms rely on exploration to discover new behaviors, which is typically achieved by following a stochastic policy. In continuous control tasks, policies with a Gaussian distribution have been widely adopted. Gaussian exploration however does not result in smooth trajectories that generally correspond to safe and rewarding behaviors in practical tasks. In addition, Gaussian policies do not result in an effective exploration of an environment and become increasingly inefficient as the action rate increases. This contributes to a low sample efficiency often observed in learning continuous control tasks. We introduce a family of stationary autoregressive (AR) stochastic processes to facilitate exploration in continuous control domains. We show that proposed processes possess two desirable features: subsequent process observations are temporally coherent with continuously adjustable degree of coherence, and the process stationary distribution is standard normal. We derive an autoregressive policy (ARP) that implements such processes maintaining the standard agent-environment interface. We show how ARPs can be easily used with the existing off-the-shelf learning algorithms. Empirically we demonstrate that using ARPs results in improved exploration and sample efficiency in both simulated and real world domains, and, furthermore, provides smooth exploration trajectories that enable safe operation of robotic hardware. 
\end{abstract}

\section{Introduction}
Reinforcement Learning (RL) is a promising approach to solving complex real world tasks with physical robots, supported by recent successes \cite{andry2018learning,pmlr-v87-kalashnikov18a,haarnoja2018soft}. Exploration is an integral part of RL responsible for discovery of new behaviors. It is typically achieved by executing a stochastic behavior policy \cite{sutton2018reinforcement}. In continuous control domain, for instance, policies with a parametrized Gaussian distribution have been commonly used \cite{schulman2015trust,levine2016end,mnih2016asynchronous,schulman2017proximal,haarnoja2018soft}. The samples from such policies are temporally coherent only through the distribution mean. In most environments this coherence is not sufficient to provide consistent and effective exploration. In early stages of learning in particular, with randomly initialized policy parameters, exploration essentially relies on a white noise process around zero mean. In environments where actions represent low-level motion control, e.g. velocity or torque, such exploration rarely produces a consistent motion that could lead to discovery of rewarding behaviors. This contributes to low sample efficiency of learning algorithms \cite{wawrzynski2015control,van2017generalized,plappert2018multi}. For real world robotic applications a short reaction time is often desirable, however, as action rate increases, a white noise exploration model becomes even less viable, effectively locking the robot in place \cite{plappert2018multi}. In addition, temporally incoherent exploration results in a jerky movement on physical robots leading to hardware damage and safety concerns \cite{peters2007reinforcement}. 

 In this work we observe that the parameters of policy distribution typically exhibit high temporal coherence, particularly at higher action rates. Specifically, in the case of Gaussian policy distribution, the action can be represented as a sum of deterministic parametrized mean and a scaled white noise component. The mean part typically changes smoothly between subsequent states. It is the white noise part that results in inconsistent exploration behavior. We propose to replace the white noise component with a stationary autoregressive Gaussian process that has stationary standard normal distribution, while exhibiting temporal coherence between subsequent observations. We derive a general form of these processes of an arbitrary order and show how the degree of temporal coherence can be continuously adjusted with a scalar parameter. We demonstrate an advantage of higher orders processes compared to the first order ones used in prior work. Further, we propose an agent's policy structure that directly implements the autoregessive computation in such processes. Temporal action smoothing mechanism therefore is not hidden from the agent but is made explicit through its policy function. In order to achieve this, we require a fixed length history of past states and actions. However, the set of resulting history-dependent policies contains the set of Markov deterministic policies, and those contain the optimal policies in many tasks \cite[Section 4.4]{puterman2014markov}. We find that, in practical applications, \textit{the search} for such optimal policies can be more efficient and safe in a space of history-dependent stochastic policies with special structure, compared to conventional search in a space of Markov stochastic policies. 

Empirically we show that proposed autoregressive policies can be used with off-the-shelf learning algorithms and result in superior exploration and learning in sparse reward tasks compared to conventional Gaussian policies, while achieving similar or slightly better performance in tasks with dense reward. We also show that the drop in learning performance due to increasing action rate can be greatly mitigated by increasing the degree of temporal coherence in the underlying autoregressive process. In the real world robotic experiments we demonstrate that the autoregressive policies result in a  smoother and safer movement. \footnote{See accompanying video at \url{https://youtu.be/NCpyXBNqNmw}}

\section{Related Work}

The problem of exploration in Reinforcement Learning has been studied extensively. One approach has been to modify the environment by changing its reward function to make it easier for an agent to obtain any rewards or to encourage the agent to visit new environment states. This approach includes work on reward shaping \cite{ng1999policy} and auxiliary reward components such as curiosity \cite{oudeyer2007intrinsic,houthooft2016vime,pathak2017curiosity,burda2018large}. Note that regardless of chosen reward function temporally consistent behavior would still be beneficial in most tasks as it would discover rewarding behaviors more efficiently. A randomly initialized agent is unaware of the reward function and for example will not exhibit curiosity until after some amount of learning, which already requires visiting new states and discovering rewarding behaviors in the first place.

A second approach, particularly common in practical robotic applications, has been to directly enforce temporal coherence between subsequent motion commands. In the most simple case a low-pass filter is applied, e.g. the agent actions are exponentially averaged over the fixed or infinite length window \cite{benbrahim1997biped}. A similar alternative is to employ a derivative control where agent's actions represent higher order derivatives of the control signal \cite{mahmood2018setting}. Both of these approaches correspond to acting in a modified MDP with different state and action spaces and result in a less direct connection between agent's action and its consequence in the environment, which can make the learning problem harder. They also make the process less observable unless the agent has access to the history of past actions used in smoothing and to the structure of a smoothing mechanism itself, which is typically not the case. 
As in the case with modified reward function, the optimal policies in the new MDP generally may not correspond to the optimal policies in the original MDP. 

A third approach has been to learn parameters of predefined parametrized controllers, such as motor primitives, instead of learning control directly in the actuation space \cite{peters2008reinforcement}. This approach is attractive, as it allows to ensure safe robot behaviour and often results in an easier learning problem \cite{van2017generalized}. However, it requires expert knowledge to define appropriate class of controllers and limits possible policies to those, representable within the selected class. In complex tasks \cite{andry2018learning,pmlr-v87-kalashnikov18a} it may be non-trivial to design a sufficiently rich primitives set.

Several studies have considered applying exploration noise to policy distribution parameters such as network weights and hidden units activations. Plappert \textit{et al.} \shortcite{plappert2017parameter} applied Gaussian exploration noise to policy parameters at the beginning of each episode, demonstrating a more coherent and efficient exploration behavior compared to only adding Gaussian noise to the action itself. Fortunato \textit{et al.} \shortcite{fortunato2017noisy} similarly applied independent Gaussian noise to policy parameters, where the scale of the noise was also learned via gradient descent. Both of these works demonstrated improved learning performance compared to baseline Gaussian action space exploration, in particular in tasks with sparse rewards. Our approach is fully complimentary to auxiliary rewards and parametric noise ideas, as both still rely on exploration noise in the action space in addition to other noise sources and can benefit from consistent and temporally smooth exploration trajectories provided by our method.

In the context of continuous control deep RL our work is most closely related to the use of temporally coherent Gaussian noise during exploration. Wawrzynski \shortcite{wawrzynski2015control} used moving average process for exploration where temporal smoothness of exploration trajectories was determined by the integer size of an averaging window. They showed that learning with such process results in a similar final performance as with standard Gaussian exploration, while providing smoother behavior suitable for physical hardware applications. Hoof~\textit{et~al.}~\shortcite{van2017generalized} proposed a stationary first order AR exploration process in parameters space.
Lillicrap \textit{et al.} \shortcite{lillicrap2015continuous} and Tallec \textit{et al.} \shortcite{tallec2019making} used Ornstein–Uhlenbeck (OU) process for exploration in off-policy learning. The latter work showed that adjusting process parameters according to the time step duration helps to maintain exploration performance at higher action rates. It can be shown that in a discrete time form OU process is a first order Gaussian AR process, which makes it a particular case of our model. AR processes derived in this work generalize the processes used in these studies, providing a wider space of possible exploration trajectories. In addition, the current work proposes policy structure that directly implements autoregressive computation, in contrast to the above studies, where the agent was unaware of the noise structure. Due to this explicit policy formulation, autoregressive exploration can be used in both, on-policy and off-policy learning.

Autoregressive architectures have been proposed in the context of high-dimensional discrete or discretized continuous action spaces \cite{metz2017discrete,vinyals2017starcraft} with regression defined over action components. The objective of such architectures was to reduce dimensionality of the action space. In contrast, we draw on autoregressive stochastic processes literature, and define regression over time steps directly in a multidimensional continuous action space with the objective of enforcing temporally coherent behaviour. 

\section{Background}
\label{sec:bg}
\subsection*{Reinforcement Learning}
Reinforcement Learning (RL) framework \cite{sutton2018reinforcement} describes an agent interacting with an environment at discrete  time steps. At each step $t$ the agent receives the environment state $s_t \in \mathcal{S}$ and a scalar reward signal $r_{t} \in R$. The agent selects an action $a_t \in \mathcal{A}$ according to a policy defined by a probability distribution 
$\pi(a|s) \coloneqq P\left\{ a_t=a | s_t=s \right\}$. At the next time step $t + 1$ in part due to the agent's action, the environment transitions to a new state $s_{t+1}$ and produces a new reward $r_{t+1}$ according to a transition probability distribution  $p(s^{\prime}, r|s, a) \coloneqq \Pr\left\{ s_{t+1} = s^{\prime}, r_{t+1}=r | s_t = s, a_t = a \right\}$.
The objective of the agent is to find a policy that maximizes the expected return defined as the future  accumulated rewards $G_t \coloneqq \sum_{k=t}^\infty \gamma^{k-t} r_{k+1}$, where $\gamma\in[0,1]$ is a discount factor.
In practice, the agent observes the environment's state partially through a real-valued observation vector $o_t$.

\subsection*{Autoregressive processes}

An autoregressive process of order $p \in \mathbb{N}$ (AR-$p$) is defined as

\begin{eqnarray}
\label{arp}
X_t &= \sum_{k = 1}^p \phi_k X_{t-k} + Z_t,
\end{eqnarray}
where $\phi_k \in R,\ k = 1, \dots, p$ are real coefficients, and $Z_t$ is a white noise with zero mean and finite variance,  $Z_t~\sim~ \text{WN}(0, \sigma_Z^2),\ \sigma_Z^2 < \infty$.

An autoregressive process $\{X_t\}$ is called weakly \textit{stationary}, if its mean function $\mu_X(t) = \mathbb{E}[X_t]$ is independent of $t$ and its covariance function $\gamma_X(t + h, t) = \text{cov}(X_{t + h}, X_t)$ is independent of $t$ for each $h$. In the future we will use the term stationary implying this definition. The process (\ref{arp})  is stationary if the roots $G_i, i = 1, \dots, p$ (possibly complex) of its \textit{characteristic polynomial} 

$$P(z) \coloneqq z^p - \sum_{i = 1}^p\phi_i z^{p-i}$$ lie within a unit circle, e.g. $|G_i| < 1, i = 1, \dots, p$ (see e.g. \cite[Section 3.1]{brockwell2002introduction}). 

An autocovariance function $\gamma_k$ is defined as $\gamma_k = \text{cov}(X_t, X_{t - k}),\ k = 0, \pm 1, \pm 2, \ldots$. From definition, $\gamma_0 = \text{var}(X_t) = \sigma_X^2$. For a stationary AR-$p$ process a linear system of Yule-Walker equations holds:

\begin{equation}
\label{yw}
\begin{aligned}
& \begin{bmatrix}
\gamma_1 \\
\gamma_2 \\
\gamma_3 \\
\vdots \\
\gamma_p
\end{bmatrix}  = 
\begin{bmatrix}
\gamma_0 & \gamma_1 & \ldots & \gamma_{p-1} \\
\gamma_1 & \gamma_0 & \ldots & \gamma_{p-2} \\
\gamma_2 & \gamma_1 & \ldots & \gamma_{p-3} \\
\vdots & \vdots & \ddots & \vdots \\
\gamma_{p-1} & \gamma_{p-2} & \ldots & \gamma_{0} \\
\end{bmatrix} 
\begin{bmatrix}
\phi_1 \\
\phi_2 \\
\phi_3 \\
\vdots \\
\phi_p
\end{bmatrix}\\
\text{and}& \\
& \gamma_0 = \sum_{i = 1}^p \phi_i \gamma_i + \sigma_Z^2.
\end{aligned}
\end{equation}

The system (\ref{yw}) has a unique solution with respect to the variables $\{\gamma_k\},\ k = 0, \ldots, p$.
\section{Stationary autoregressive Gaussian processses}

In this section we derive a family of stationary AR-$p$ Gaussian processes for any $p \in \mathbb{N}$, such that $X_t \sim \mathcal{N}(0, 1)\ \forall t$, meaning $X_t$ has a marginal standard normal distribution at each $t$. We also show how the degree of temporal smoothness of trajectories formed by subsequent observations of such processes can be continuously tuned with a scalar parameter.

\begin{theorem}
\label{theorem}
For any $p \in \mathbb{N}$ and for any $\alpha_k \in [0, 1),\\ k~=~1, \dots, p$ consider a set of coefficients
\begin{equation}
\label{coeffs}
\begin{aligned}
\{\tilde{\phi}_k\}_{k=0}^p &= (-1)^{k + 1}\sum\limits_{1 
\le i_1 < i_2 \dots < i_k \le p} \alpha_{i_1} \alpha_{i_2} \dots \alpha_{i_k}.
\end{aligned}
\end{equation}
The Yule-Walker system (\ref{yw}) with coefficients $\{\tilde{\phi}_k\}$ has a unique solution with respect to $(\tilde{\gamma}_0, \tilde{\gamma}_1, \ldots, \tilde{\gamma}_p,\tilde{\sigma}_Z^2)$, such that $\tilde{\gamma}_0 = 1$ and $\tilde{\sigma}_Z^2 > 0$.  
Furthermore, the autoregressive process
\begin{equation}
\label{arpn}
\begin{aligned}
X_t &= \sum_{k = 1}^p \tilde{\phi}_k X_{t-k} + Z_t\\
Z_t &\sim \mathcal{N}(0, \tilde{\sigma}_Z^2),
\end{aligned}
\end{equation}
is a stationary Gaussian process with zero mean and unit variance, meaning $X_t \sim \mathcal{N}(0, 1) \ \forall t$.
\end{theorem}

\begin{proof}
The proof follows from the observation that $\{\tilde{\phi}_k\}$ are coefficients of a polynomial $P(z) = (z - \alpha_1)(z - \alpha_2)\ldots(z - \alpha_p)$ with roots $\{\alpha_k\}$ that all lie within a unit circle. Since $P(z)$ is a characteristic polynomial of a process ($\ref{arpn}$), the process is stationary. The existence of a unique solution to the system (\ref{yw}) with $\tilde{\gamma}_0 = 1,\ \tilde{\sigma}_Z^2 > 0$ follows from the observation that for any $\tilde{\sigma}_Z^2 > 0$ the system (\ref{yw}) has a unique solution with respect to $(\tilde{\gamma}_0, \tilde{\gamma}_1, \ldots, \tilde{\gamma}_p)$, while it is homogeneous with respect to $(\tilde{\gamma}_0, \tilde{\gamma}_1, \ldots, \tilde{\gamma}_p,\tilde{\sigma}_Z^2)$. The result then follows from observing that $\tilde{\gamma}_0 > 0$, and scaling the solution $(\tilde{\gamma}_0, \tilde{\gamma}_1, \ldots, \tilde{\gamma}_p,\tilde{\sigma}_Z^2)$ by $1/\tilde{\gamma}_0$. A complete proof can be found in Appendix \ref{app:proof}.
\end{proof}

\begin{corollary}
For any $p \in \mathbb{N}$ and for any $\alpha \in [0, 1)$ let $\{X_t\}$ be the autoregressive process:
\begin{equation}
\label{arpn2}
\begin{aligned}
X_t &= \sum_{k = 1}^p \tilde{\phi}_k X_{t-k} + Z_t\\
\tilde{\phi}_k &= (-1)^{k + 1}{p\choose k} \alpha^k\\ 
Z_t &\sim \mathcal{N}(0, \tilde{\sigma}_Z^2),
\end{aligned}
\end{equation}
where ${p\choose k} = \frac{p!}{k!(p - k)!}$ is a binomial coefficient, $\tilde{\sigma}_Z^2$ is a solution to the system (\ref{yw}) with $\{\phi_k = \tilde{\phi}_k\}_{k=1}^p$ and $\gamma_0=1$. Then $X_t \sim \mathcal{N}(0, 1) \ \forall t$.
\end{corollary}
\begin{proof}
Set $\alpha_k = \alpha \ \forall k$ in (\ref{coeffs}) and apply Proposition \ref{theorem}.
\end{proof}
\noindent An example of a third order process AR-3 defined by (\ref{arpn}) is given in Appendix \ref{app:ar3}.

Proposition \ref{theorem} allows to formulate stationary autoregressive processes of an arbitrary order $p$ for arbitrary values $\alpha_k \in [0, 1), k = 1, \dots, p$, such that the marginal distributions of realizations of these processes are standard normal at each time step. This gives us great flexibility and power in defining properties of these processes, such as the degree of temporal coherence between process realizations at various time lags. If we were to use these processes as a source of exploration behavior in RL algorithms, this flexibility would translate into a flexibility in defining the shape and smoothness of exploration trajectories. Note, that the process (\ref{arpn}) trivially generalizes to a vector form by defining $Z_t$ a multivariate white noise with a diagonal covariance.

Autoregressive processes in the general form (\ref{arpn}) possess a number of interesting properties that can be utilized in reinforcement learning. However, for the purposes of the discussion in the following sections, from now on we will consider a simpler subfamily of processes, defined by (\ref{arpn2}). Notice, that $\alpha = 0$ results in $\tilde{\phi}_k = 0\ \forall k$, and $X_t$ becomes a white Gaussian noise. On the other hand, $\alpha \rightarrow 1$ results in $\sigma_Z^2 \rightarrow 0$, and $X_t$ becomes a constant function. Therefore, tuning a single scalar parameter $\alpha$ from $0$ to $1$ continuously adjusts temporal smoothness of $X_t$ ranging from white noise to a constant function. 
Figure \ref{fig:noise_fig} shows realizations of such processes at different values of $p$ and $\alpha$. 
\begin{figure}[t]
  \centering
  \includegraphics[width=0.48\textwidth]{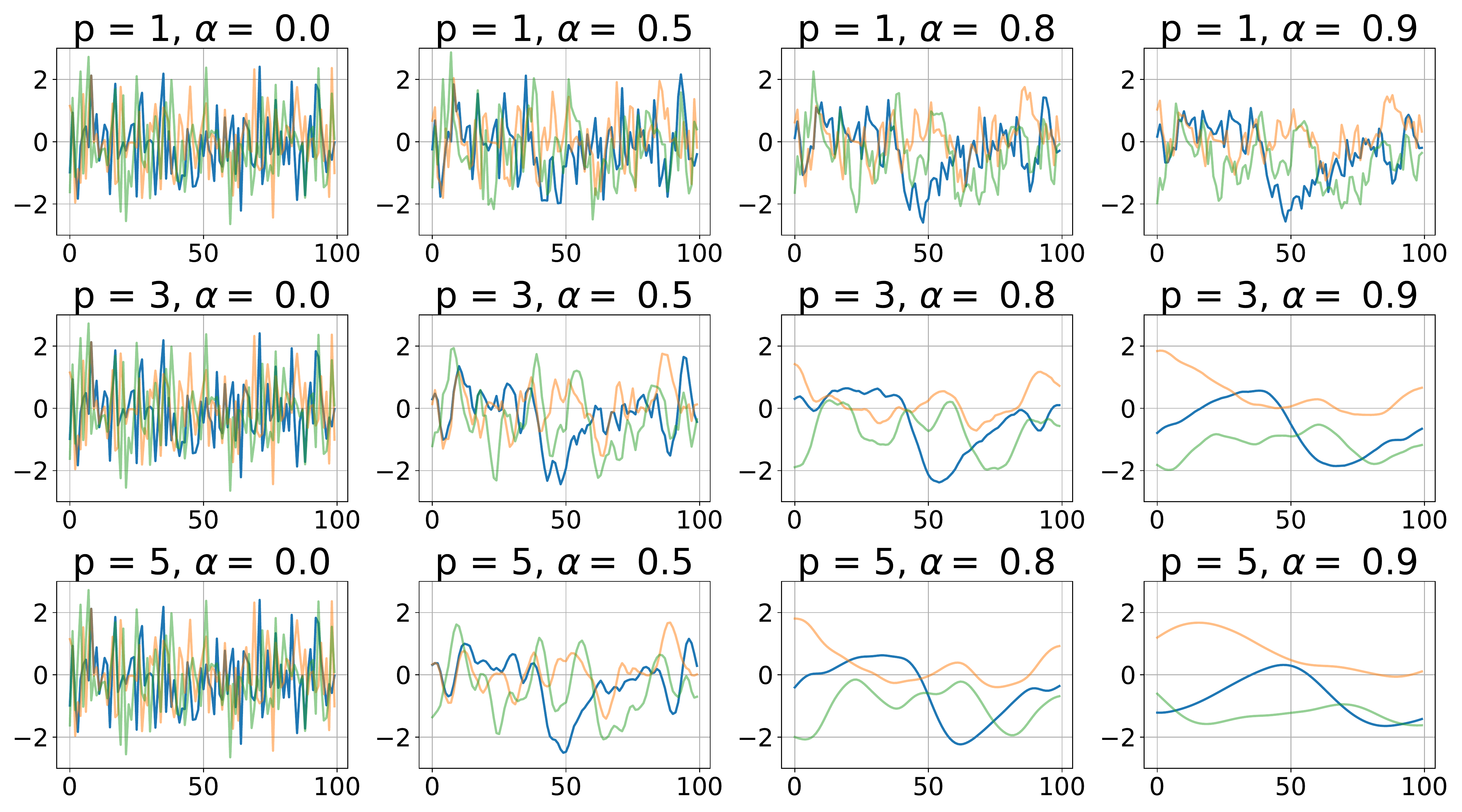}
  \caption{Realizations of processes (\ref{arpn2}) for different $p$ and $\alpha$ and the same set of 3 random seeds.}
  \label{fig:noise_fig}
\end{figure}
The realizations are initialized from the same set of 3 random seeds for each $p$, $\alpha$ pair. 

\section{Autoregressive policies}
In continuous control RL a policy is often defined as a parametrized diagonal Gaussian distribution: 
\begin{equation}
\label{gauss}
p_{\theta}(a_t | s_t) = \mathcal{N}(\mu_{\theta}(s_t), \sigma_{\theta}^2(s_t)\cdot I),
\end{equation}
where $s_t$ is a state at time $t$, $\mu_{\theta}(s_t)$ and $\sigma_{\theta}^2(s_t)$ are vectors parametrized by deep neural networks.
The actions, sampled from such distribution, can be represented as $a_t~=~\mu_{\theta}(s_t) + \sigma_{\theta}(s_t)\varepsilon_t,$ where $\varepsilon_t \sim \mathcal{N}(0, I)$ is a white Gaussian noise.
We propose to replace $\varepsilon_t$ with observations of an AR-$p$ process $\{X_t\}$ defined by (\ref{arpn2}) for some $p \in N$ and $\alpha \in [0, 1)$:
\begin{equation}
\label{action}
a_t = \mu_{\theta}(s_t) + \sigma_{\theta}(s_t)X_t.
\end{equation}
Both $\varepsilon_t$ and $X_t$ follow marginal standard normal distribution at each step $t$, therefore such substitution does not change the network output to noise ratio in sampled actions, however for $\alpha > 0$ the sequence $\{X_t\}$ possesses temporal coherence and can provide a more consistent exploration behavior. We~would like to build an agent that implements stochastic policy with samples, defined by (\ref{action}).
From definition (\ref{arpn2}) of the process $\{X_t\}$, (\ref{action}) can be expanded as
\begin{equation}
\label{ac_rep}
\begin{aligned}
a_t&=\mu_{\theta}(s_t) + \sigma_{\theta}(s_t)\sum_{k = 1}^p \tilde{\phi}_k X_{t-k} + \sigma_{\theta}(s_t)\tilde{\sigma}_Z \varepsilon_t,\\
\varepsilon_t &\sim \mathcal{N}(0, I).
\end{aligned}
\end{equation}
From (\ref{action}) also, $X_{t} = (a_{t} - \mu_{\theta}(s_t))/\sigma_{\theta}(s_t)\ \forall t$,
hence (\ref{ac_rep}) can be rewritten as
\begin{equation}
\label{ac_rep2}
\begin{aligned}
a_t &=\mu_{\theta}(s_t) + \sigma_{\theta}(s_t)\sum_{k = 1}^p \tilde{\phi}_k \frac{a_{t-k} - \mu_{\theta}(s_{t-k})}{\sigma_{\theta}(s_{t-k})} +\\
&+ \sigma_{\theta}(s_t)\tilde{\sigma}_Z \varepsilon_t,\\
\varepsilon_t &\sim \mathcal{N}(0, I).
\end{aligned}
\end{equation}
Denote $f_{\theta, t} = \sum_{k = 1}^p \tilde{\phi}_k \frac{a_{t-k} - \mu_{\theta}(s_{t-k})}{\sigma_{\theta}(s_{t-k})}$ the auto-regressive "history" term in (\ref{ac_rep2}). Note that $f_{\theta, t}$ is a function of past $p$ states and actions, $f_{\theta, t} = f(\{s_{t-k}, a_{t-k}\}_{k=1}^p, \theta)$. Then $a_t$ follows the distribution:
\begin{equation}
\label{ac_dist}
\begin{aligned}
a_t &\sim \mathcal{N}(\mu_{\theta}(s_t) + \sigma_{\theta}(s_t) f_{\theta, t}, \sigma_{\theta}^2(s_t)\tilde{\sigma}^2_Z \cdot I),\\
\tilde{\phi}_k&,\ \tilde{\sigma}^2_Z \text{ are defined by (\ref{arpn2}).}
\end{aligned}
\end{equation}

In order to implement such action distribution, we need to define a \textit{history-dependent} policy $\pi(a_t|s_t, h^p_t)$, where $h^p_t = (s_{t-p}, a_{t-p}, \ldots, s_{t-1}, a_{t-1})$ is a history of past $p$ states and actions. In general, history-dependent policies do not induce Markov stochastic processes, even if the environment transition probabilities are Markovian \cite[Section 2.1.6]{puterman2014markov}. However when the dependence is only on a history of a fixed size, such policy induces a Markov stochastic process in an extended state space, where states are defined as pairs $(h^p_t, s_t)$. In order to be able to lean on existing theoretical results, such as Policy Gradient Theorem \cite{sutton2000policy}, and to use existing learning algorithms, we will talk about learning policies in this extended MDP. 

More formally, let $M = (S, A, P(\cdot|a, s), r(s, a))$ be a given MDP with $S$ and $A$ denoting state and action sets, and $P$ and $r$ denoting transition probability and reward functions respectively. Let $p$ be an arbitrary integer number. We define a modified MDP $\tilde{M}^{p} = (\tilde{S}, \tilde{A}, \tilde{P}(\cdot |a, \tilde{s}), \tilde{r}(\tilde{s}, a))$ with the elements $\tilde{S} = \{S \times A\}^p \times S$, where $\{C\}^p$ denotes Cartesian product of set $C$ with itself $p$ times, $\tilde{A} = A$, and $\tilde{P}$ and $\tilde{r}$ defined as follows:

\begin{align*}
\forall \tilde{s}, \tilde{s}^{\prime} &\in \tilde{S}:\\
\tilde{s} &= (s_{1}, a_1, \ldots, s_p, a_p, s_{p+1}), a_k \in A, s_k \in S\ \forall k\\
\tilde{s}^{\prime} &= (s^{\prime}_{1}, a^{\prime}_1, \ldots,  s^{\prime}_p, a^{\prime}_p, s^{\prime}_{p+1}), a^{\prime}_k \in A, s^{\prime}_k \in S\ \forall k\\
\tilde{P}(\tilde{s}^{\prime} |a, \tilde{s}) &= \begin{cases} P(s^{\prime}_{p+1}|a, s_{p+1}), & \text{if }
\!\begin{aligned}[t]
       s^{\prime}_k &= s_{k+1}, k \le p\\
       a^{\prime}_k &= a_{k+1},\ k < p\\
       a^{\prime}_p &= a
       \end{aligned}
       \\
& \\
0 &  \text{otherwise} \end{cases}\\
\tilde{r}(\tilde{s}, a) &= r(s_{p+1}, a)
\end{align*}

In other words, transitions in a modified MDP $\tilde{M}^{p}$ correspond to transitions in the original MDP $M$ with states in $\tilde{M}^p$ containing also the history of past $p$ states and actions in $M$. 
The interaction between the agent and the environment, induced by $\tilde{M}^p$, occurs in the following way. At each time $t$ the agent is presented with the current state $\tilde{s}_t = (s_{t-p}, a_{t-p}, \ldots, s_{t-1}, a_{t-1}, s_t)$. Based on this state and its policy, it chooses an action $a_t$ from the set $A$ and sends it to the environment. Internally, the environment propagates the action $a_t$ to the original MDP $M$, currently in state $s_t$, which responds with a reward value $r_{t+1}$ and transitions to a new state $s_{t+1}$. At this moment, the MDP $\tilde{M}^p$ transitions to a new state $\tilde{s}_{t+1} = (s_{t-p +1}, a_{t-p+1}, \ldots, s_{t}, a_{t}, s_{t+1})$ and presents it to the agent together with the reward $r_{t+1}$. Let $s_0$ be an element of the set of the initial states of $M$. A corresponding initial state of $\tilde{M}^p$ is defined as $\tilde{s}_0 = (\underbrace{s_0, a_0, \ldots, s_0, a_0}_\text{p \text{ repetitions}}, s_0)$, where $a_0$ is any element of an action set $A$, for example zero vector in the case of continuous space. The particular choice of $a_0$ is immaterial, since it does not affect future transitions and rewards (more details in Appendix \ref{app:details}).

\noindent We define an \textbf{autoregressive policy} (ARP) over $\tilde{M}^p$ as:
\begin{equation}
\label{policy}
\begin{aligned}
\forall \tilde{s}_t &= (s_{t-p}, a_{t-p}, \ldots, s_{t-1}, a_{t-1}, s_t): \\
\pi_{\theta}(a_t|\tilde{s}_t) &= \mathcal{N}(\mu_{\theta}(s_t) + \sigma_{\theta}(s_t) f_{\theta}(\tilde{s}_t),\ \sigma^2_{\theta}(s_t)\tilde{\sigma}^2_Z I),\\
f_{\theta}(\tilde{s}_t) &= \sum_{k = 1}^p \tilde{\phi}_k \frac{a_{t-k} - \mu_{\theta}(s_{t-k})}{\sigma_{\theta}(s_{t-k})},\\
\tilde{\phi}_k,\ \tilde{\sigma}^2_Z &\text{ are defined by (\ref{arpn2})},
\end{aligned}
\end{equation}
where $\mu_{\theta}(\cdot)$ and $\sigma_\theta(\cdot)$ are parametrized function approximations, such as deep neural networks. For notation brevity, we omitted dependence of policy $\pi_{\theta}$ on $\{\tilde{\phi}_k\}$ and $\tilde{\sigma}_Z$, since these values are constant once the autoregressive model $\{X_t\}$ is selected. In this parametrization, $\mu_{\theta}(s_{t - k}),\ k = 0, \ldots, p$ should be thought of as the same parametrized function $\mu_{\theta}(\cdot)$ applied to different parts of the state vector $\tilde{s}_t$, therefore each occurence of $\mu_{\theta}(\cdot)$ in (\ref{policy}) contributes to the gradient w.r.t. parameters $\theta$. Similarly, each occurence of $\sigma_{\theta}(\cdot)$ contributes to the gradient w.r.t. $\theta$. Note, that including history of states and actions does not affect the dimensionality of the input to the function approximations, as both $\mu_{\theta}(\cdot)$ and $\sigma_{\theta}(\cdot)$ accept only states from the original space as inputs.

The history-dependent policy (\ref{policy}) results in the desired action distribution (\ref{ac_dist}) in the original MDP $M$, at the same time with respect to $\tilde{M}^p$ it is just a particular case of a Gaussian policy (\ref{gauss}). Formally, we will perform learning in $\tilde{M}^p$, where $\pi_{\theta}$ is Markov, and therefore all the related theoretical results apply, and any off-the-shelf learning algorithm, applicable to policies of type (\ref{gauss}), can be used. In particular, the value function in e.g. actor-critic architectures is learned with usual methods. Empirically we found that conditioning value function only on a current state $s_t$ from the original MDP instead of an entire vector $\tilde{s}_t$ gives more stable learning performance. It also helps to maintain the critic network size invariant to the AR process order $p$.

By design, for each sample path $(\tilde{s}_0, a_0, \tilde{s}_1, a_1, \ldots)$ in $\tilde{M}^p$ there is a corresponding sample path $(s_0, a_0, s_1, a_1, \ldots)$ in $M$ with identical rewards. Therefore, improving the policy and the obtained rewards in $\tilde{M}^p$ results in identical improvement of a corresponding history-dependent policy in $M$. Notice also, that if $\sigma_{\theta}(s_t) \rightarrow 0$ in (\ref{policy}), then $\pi_{\theta}$ reduces to a Markov deterministic policy $a_t = \mu_{\theta}(s_t)$ in $M$. Therefore, the optimal policy in the set of ARPs defined by (\ref{policy}) is  at least as good, as the best deterministic policy in the set of policies  $a_t = \mu_{\theta}(s_t)$. This is in contrast with action averaging approaches, where temporal smoothing is typically imposed on the entire action vector and not just on the exploration component, limiting the space of possible deterministic policies. 

It is important to point out that for any history-dependent policy there exists an equivalent Markov stochastic policy with identical expected returns \cite[Theorem 5.5.1]{puterman2014markov}. For the policy (\ref{policy}), for example, it can be constructed as $\pi_{\theta}^{M}(a | s) = \sum\limits_{h^p \in H^p} \pi_{\theta}(a | s, h^p)p(h^p|s, \pi_{\theta})\ \forall (a, s)$, where $H^p$ is a set of all histories of size $p$. However, $\pi_{\theta}^{M}(a | s)$ is a non-trivial function of a state $s$, unknown to us at the beginning of learning. It is certainly not given by a random initialization of (\ref{gauss}), while a random initialization of (\ref{policy}) already provides consistent and smooth behavior. $\pi^{M}(a | s)$ also cannot be derived analytically from (\ref{policy}), since computing $p(h^p|s, \pi_{\theta})$ requires knowledge of environment transition probabilities, which we cannot expect to have for each given task. From these considerations, the particular form of policy parametrization defined by (\ref{policy}) can also be thought of as an additional structure, enforced upon the general class of Markov policies, such as policies defined by (\ref{gauss}), restricting possible behaviors to temporally coherent ones.

Although autoregressive term $f_{\theta}(\tilde{s}_t)$ in (\ref{policy}) is formally a part of the distribution mean, numerically it corresponds to a stationary zero mean random process $F_t = \sum_{k=1}^p \tilde{\phi}_k X_{t-k}$, where $\{X_t\}$ is an underlying AR process defined by (\ref{arpn2}). Therefore, $f_{\theta}(\tilde{s}_t)$ can be thought of as a part of an action exploration component around the "true" mean, given by $\mu_{\theta}(s_t)$. It is this part that ensures a consistent and smooth exploration, as will be demonstrated in the next section.

In principle, one could define $\pi_\theta$ in (\ref{policy}) using arbitrary values of coefficients $\{\tilde{\phi}_k\}$ and $\tilde{\sigma}_Z^2$. The role of particular values of $\{\tilde{\phi}_k\}$ computed according to (\ref{arpn2}) is to make sure, that the underlying AR process $\{X_t\}$ is stationary and the autoregressive part $f_{\theta}(\tilde{s}_t)$ does not explode. The role of $\tilde{\sigma}_Z^2$ computed by solving (\ref{yw}) with coefficients $\{\phi_k = \tilde{\phi}_k\}$ and $\gamma_0=1$ is to make sure, that the variance of the underlying process~$\{X_t\}$~is~1. The total variance around $\mu_{\theta}(s_t)$ is then conveniently defined by an agent controlled $\sigma_{\theta}(s_t)$.

Since linear system (\ref{yw}) with coefficients (\ref{coeffs}) and $\gamma_0 = 1$ has a unique solution according to the Proposition \ref{theorem}, its matrix has a full rank, and therefore the system is well-determined and can be solved numerically to an arbitrary precision. In practice we solve it with numpy.linalg.solve function.

\section{Experiments}
\label{sec:exps}
We compared conventional Gaussian policy with ARPs on a set of tasks with both, sparse and dense reward functions, in simulation and the real world. In the following learning experiments we used the Open AI Baselines PPO algorithm implementation \cite{schulman2017proximal,dhariwal2017openai}. The results with Baselines TRPO \cite{schulman2015trust} are provided in Appendix \ref{app:trpo}. For each experiment we used identical algorithm hyper-parameters and neural network structures to parametrize $\mu_{\theta}$, $\sigma_{\theta}$ and the value networks for both Gaussian and ARP policies. We used the same set of random seeds to initialize neural networks and the same set of random seeds to initialize environments that involve uncertainty. Detailed parameters for each task are included in Appendix \ref{app:params}.  We did not perform a hyper-parameter search to optimize for ARP performance, as our primary objective is to demonstrate the advantage of temporally coherent exploration even in the setting, tuned for a standard Gaussian policy.  The video of agent behaviors can be found at \url{https://youtu.be/NCpyXBNqNmw}. The code to reproduce experiments is available at \url{https://github.com/kindredresearch/arp}.

\subsection*{The order of an autoregressive process}
\begin{figure}[bt]
  \centering
  \includegraphics[scale=.27]{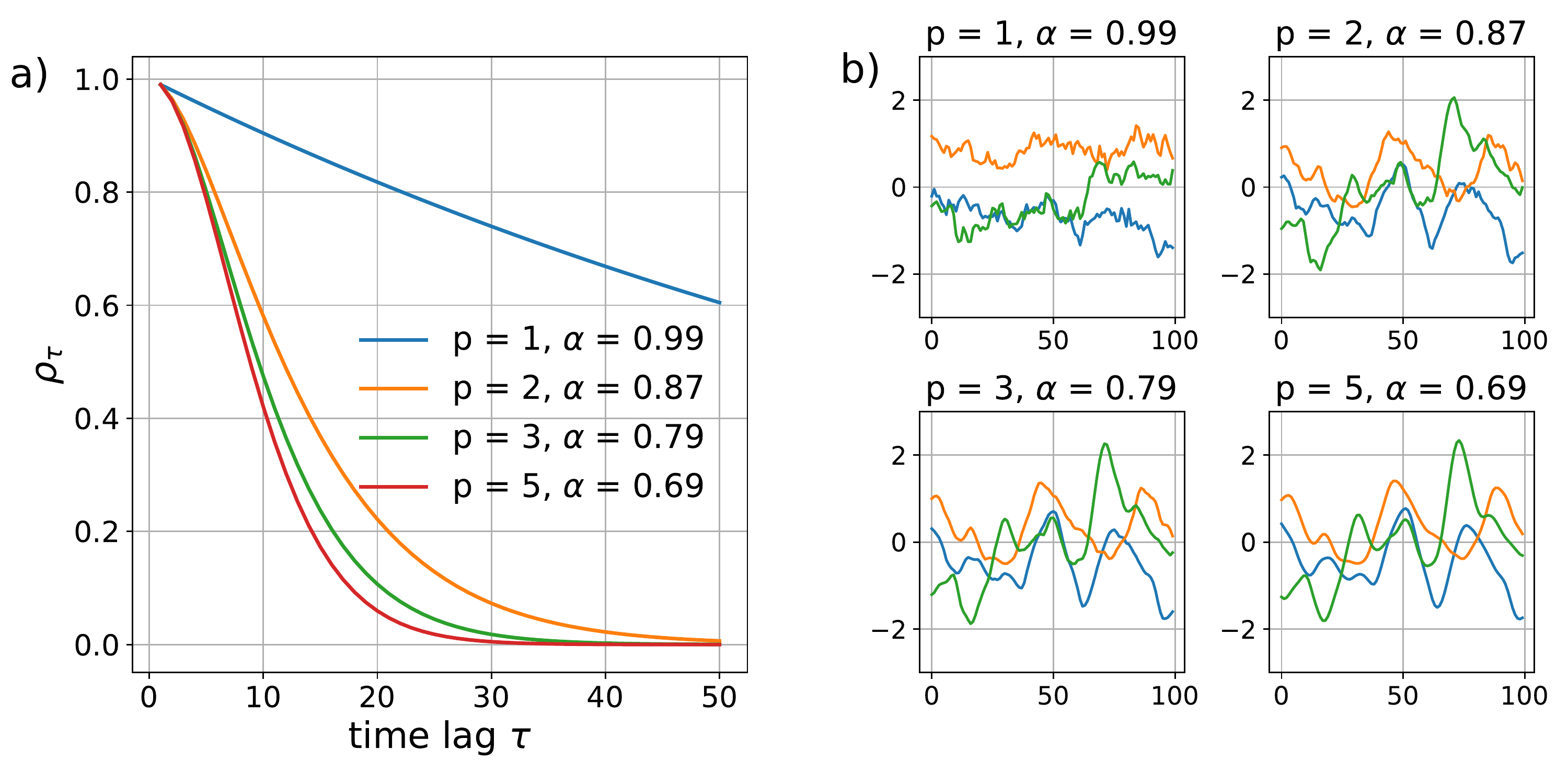}
  \caption{a) Autocorrelation function $\rho_{\tau}$ for autoregressive processes (\ref{arpn2}) with different orders $p$ but the same value of $\rho_1 = 0.99$. b) Realizations of processes (\ref{arpn2}) with the same $\rho_1 =0.99$.}
  \label{fig:acf}
\end{figure} 
From Figure \ref{fig:noise_fig} one can notice that the temporal smoothness of realizations of AR processes (\ref{arpn2}) empirically increases with both, parameter $\alpha$ and order $p$. 
Why do we need higher order processes if we can simply increase $\alpha$ to achieve a higher degree of temporal coherence? To answer this question it is helpful to consider an autocorrelation function (ARF) $\rho_{\tau} = \text{cov}(X_t, X_{t+\tau})/\text{var}(X_t) = \gamma_{\tau}/\gamma_0$ of these processes.
White Gaussian noise by definition has autocorrelation function equal to zero at any $\tau$ other than 0. An autoregressive process with non-zero coefficients generally has non-zero values of autocorrelation function at all $\tau$.

One of the reasons we are interested in autoregressive processes for exploration is that they provide smooth trajectories that do not result in jerky movement and do not damage physical robot hardware. 
Intuitively, the smoothness of the process realization is defined by a correlation between subsequent observations $\text{corr}(X_t, X_{t+1}) = \rho_{1}$, which for a given $p$ increases with increasing $\alpha$. However, given the same value $\rho_1$, processes of different orders $p$ behave differently. Figure \ref{fig:acf}a shows ARFs for different processes defined by (\ref{arpn2})  and their corresponding values of $\alpha$  with the same value of $\rho_1 = 0.99$, while Figure \ref{fig:acf}b shows realizations of these processes.
ARF values at higher orders $p$ decrease much faster with increasing time lag $\tau$ compared to the 1st order process, where correlation between past and future observations lingers over long periods of time.
As shown on Figure \ref{fig:acf}b, the 1st order AR process produces nearly a constant function, while the 5th order process exhibits a much more diverse exploratory behavior. Given the same value of correlation between subsequent realizations, higher order autoregressive processes exhibit lower correlation between observations distant in time, resulting in trajectories with better exploration potential. In robotics applications where smoothness of the trajectory can be critical, higher order autoregressive processes may be a preferable choice. 
Empirically we found that the 3-rd order  processes provide sufficiently smooth trajectories while exhibiting a good exploratory behavior, and used $p=3$ in all our subsequent learning experiments varying only the smoothing parameter $\alpha$. 
\subsection*{Toy environment with sparse reward}
\begin{figure}[tb]
  \centering
  \includegraphics[width=.5\textwidth]{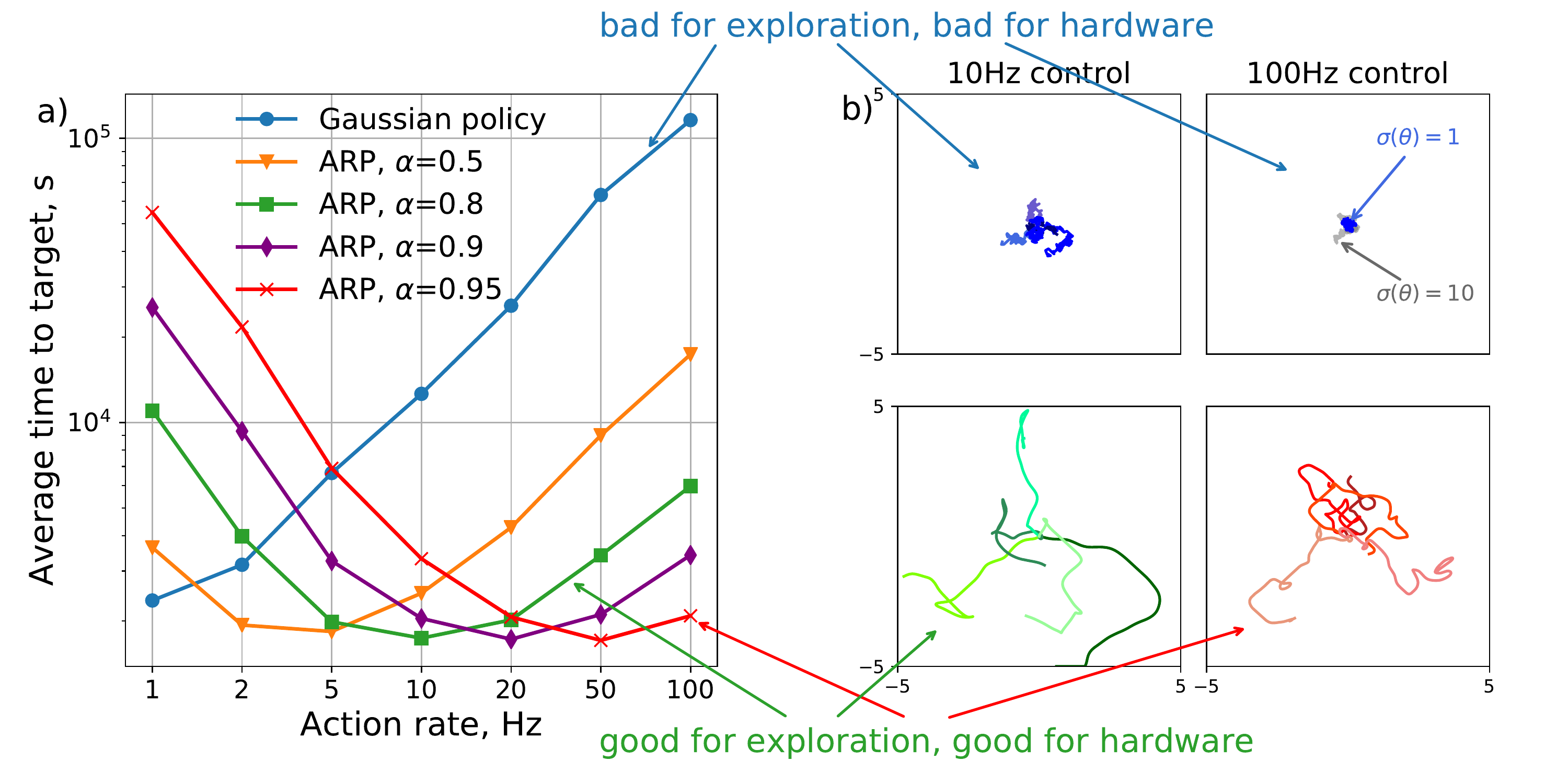}
  \caption{a) Average time to target in Square environment as a function of an action rate for Gaussian policy and ARPs with varied $\alpha$. \\b) 10 seconds long exploration trajectories at 10Hz (left column) and 100Hz (right column) action rate using Gaussian policy (top row) and ARPs with $p=3$ and $\alpha$ values 0.8 and 0.95 (bottom row).}
  \label{fig:steps}
\end{figure}  
To demonstrate the advantage of temporally consistent exploration, in particular at high action rates, we designed a toy \textit{Square} environment with a 2D continuous state space bounded by a 10x10 square arena. The agent controls a dot through a continuous direct velocity control. The agent is initialized in the middle of the arena at the start of each episode and receives a -1 reward at each time step scaled by time step duration. The target is generated at a random location on a circle of diameter 5 centered at the middle of the arena to make episodes homogenous in difficulty. The episode is over when the agent approaches the target to within a distance of 0.5. The action space is bounded within a two-dimensional $[-1, 1]^2$ interval. The observation vector contains the agent's position, velocity, and the vector difference between agent position and the target position.

To compare exploration efficiency we ran random ARP ($p=3$) and Gaussian agents with $\mu_{\theta}(\cdot),\ \sigma_{\theta}$ initialized to $\vec{0}$ and $\vec{1}$ respectively for 10 million simulated seconds at different action rates. Figure \ref{fig:steps}a shows average time to reach the target as a function of an action rate.
The results show that the optimal degree of temporal coherence depends on the environment properties, such as action rate. At low control frequency a white Gaussian exploration is more effective than ARPs with high $\alpha$, as in the latter the agent quickly reaches the boundary of the state space and gets stuck there. The efficiency of Gaussian exploration drops dramatically with the increase of action rate. However it is possible to recover the same exploration performance in ARP by increasing accordingly the $\alpha$ parameter. This effect is visualized on Figure \ref{fig:steps}b which shows five 10 seconds long exploration trajectories at 10Hz and 100Hz control for Gaussian and ARP policies. Although ran for the same amount of simulated time, Gaussian exploration at 100Hz covers substantially smaller area of state space compared to 10Hz control, while increasing $\alpha$ from 0.8 to 0.95 (the values were chosen empirically) results in ARP trajectories covering similar space at both action rates. 
Note that the issue with Gaussian policy can not be fixed by simply increasing the variance, as most actions will just be clipped at the $[-1, 1]^2$ boundary, resulting in a similarly poor exploration. Figure \ref{fig:steps}b top right plot shows exploration trajectories for $\sigma(\theta) = \vec{1}$ (blue) and $\sigma(\theta) = \vec{10}$ (gray). To the contrary of the common intuition, in bounded action spaces Gaussian exploration with high variance does not produce a diverse state-action visitation.  
\begin{figure}[tb]
\centering
  \centering
  \includegraphics[width=0.5\textwidth]{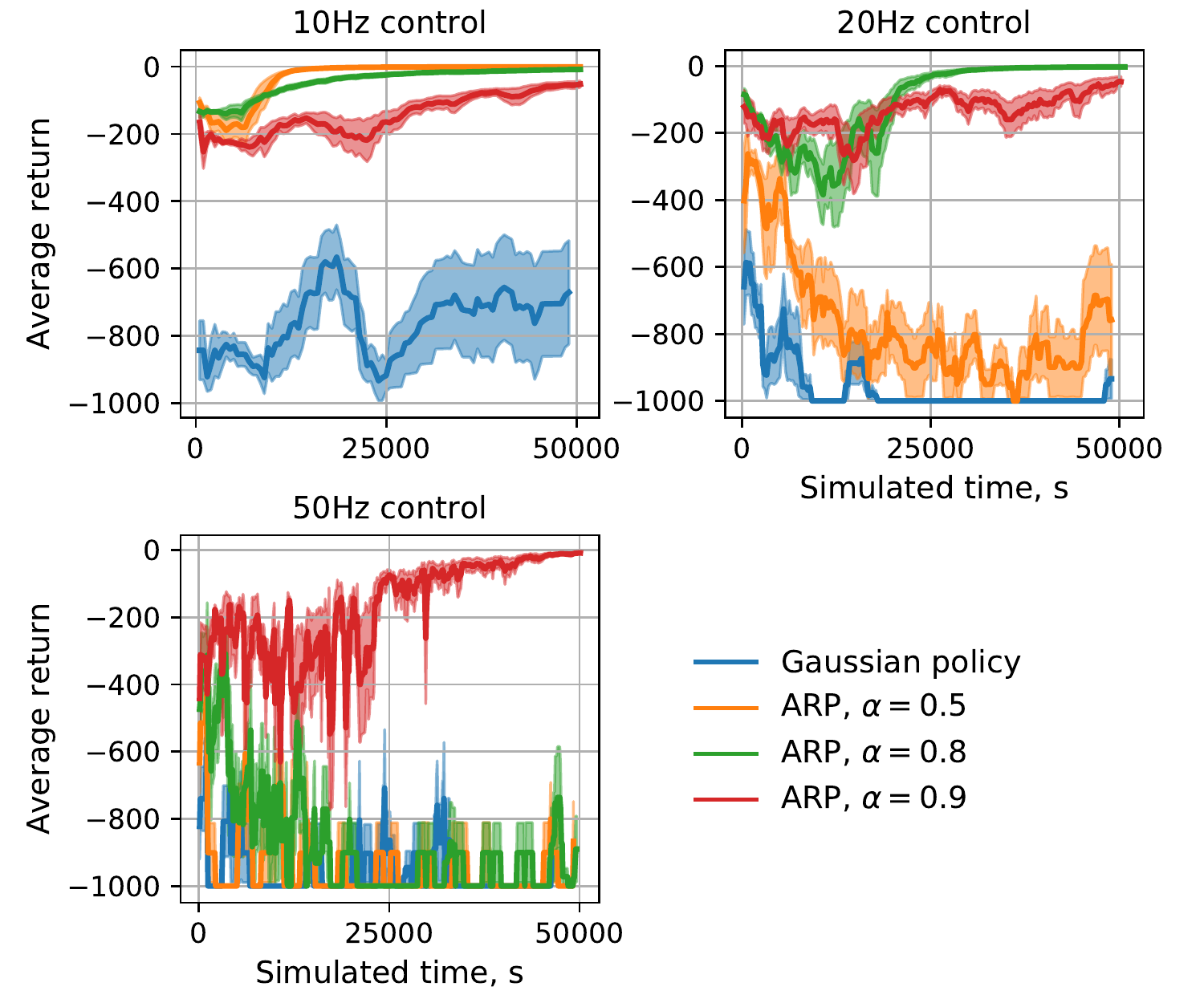}
  \caption{On a toy 2D environment with sparse reward, white noise exploration (Gaussian policy) leads to ineffective learning. Temporally smoother processes are effective at higher action rates.}
  \label{fig:sq_learn}
 \end{figure}
\begin{figure}[b]
  \centering
  \includegraphics[width=0.5\textwidth]{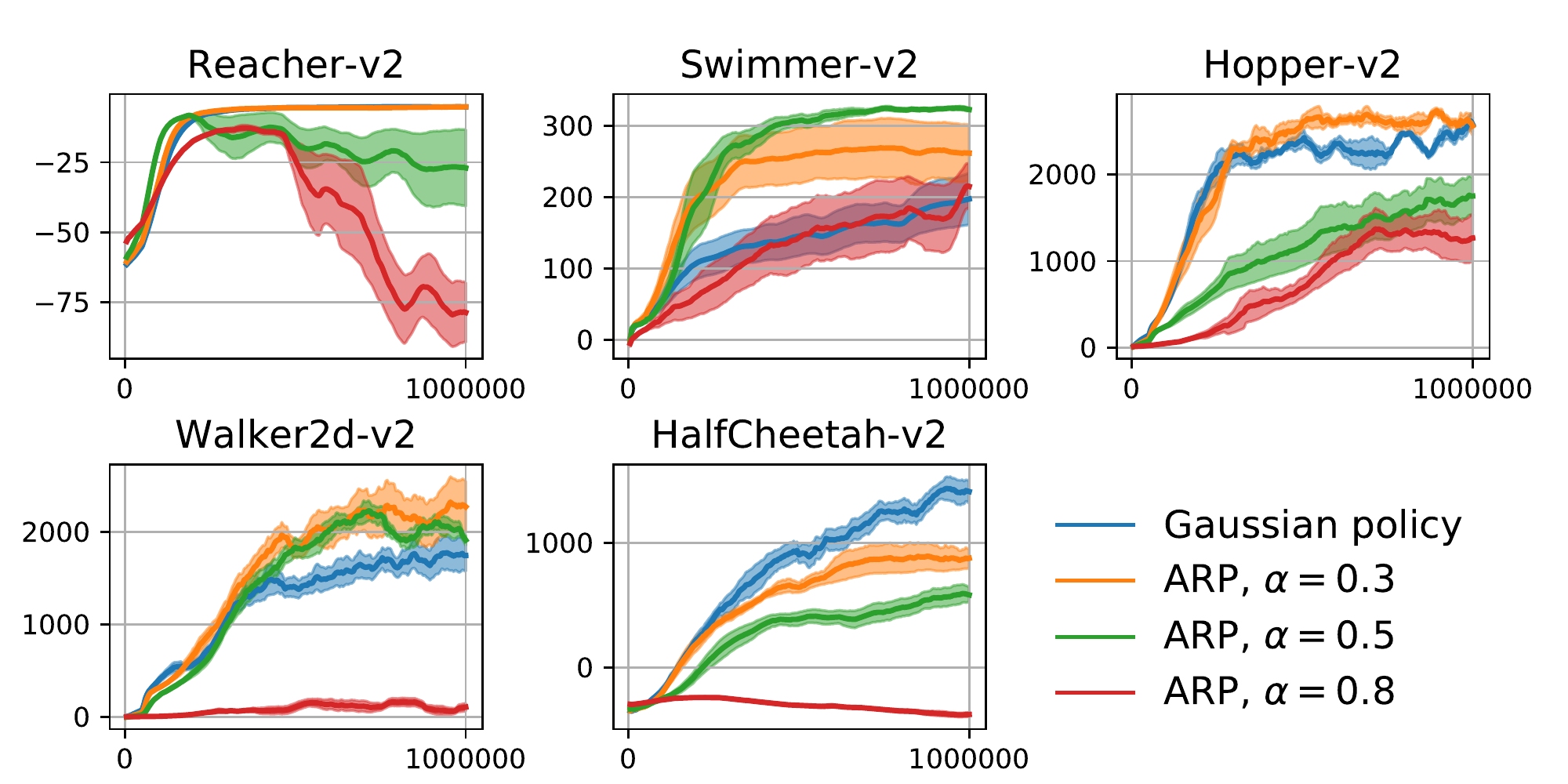}
  \caption{Learning curves in Mujoco-based environments.}
  \label{fig:mj_learn}
\end{figure}

The advantage of ARPs in exploration translates into an advantage in learning. 
Figure \ref{fig:sq_learn} shows learning curves (averaged over 5 random seeds) on Square environment at different action rates ran for 50,000 seconds of total simulated time with episodes limited to 1000 simulated seconds.
Not only ARPs exhibit better learning, but the initial random behaviour gives much higher returns compared to initial Gaussian agent behaviour. At higher action rates ARPs with higher $\alpha$ produce better results.

In the formulation of the AR-1 process used in Lillicrap \textit{et al.} \shortcite{lillicrap2015continuous} and in Tallec \textit{et al.} \shortcite{tallec2019making}, parameter $\alpha$ corresponds to $1 - \kappa dt$, where $dt$ is a time step duration. Hence, in that formulation $\alpha$ naturally approaches 1 as $dt$ approaches zero. 
Note, that in order to achieve the best performance on each given task, parameter $\kappa$ still needs to be tuned, just as parameter $\alpha$ needs to be tuned in our formulation. The optimal values of these parameters depend not only on action rate, but also on environment properties, such as a size of a state space relative to the typical size of an agent step.

\subsection*{Mujoco experiments}
Figure \ref{fig:mj_learn} shows the learning results on standard OpenAI Gym Mujoco environments \cite{1606.01540}. These environments have dense rewards, so consistent exploration is less crucial here compared to tasks with sparse rewards. Nevertheless, we found that ARPs perform similarly or slightly better, than a standard Gaussian policy. On a Swimmer-v2 environment ARP resulted in a much better performance compared to Gaussian policy, possibly because in this environment smooth trajectories are highly rewarded.

\subsection*{Physical robot experiments}
On a UR5 robotic arm we were able to obtain results similar to those in the toy environment. We designed a sparse reward version of a UR5 Reacher 2D task introduced in \cite{Mahmood2018BenchmarkingRL}. In a modified task at each time step the agent receives a -1 reward scaled by a time step duration. The episode is over when the agent reaches the target within a distance of 0.05. In order to provide sufficient time for exploration in a sparse reward setting we doubled the episode time duration to 8 seconds. Figure \ref{fig:ur5} shows the learning curves for 25Hz and 125Hz control. Each curve is an average across 4 random seeds. The Gaussian policy fails to learn in a 125Hz control setting within a 5 hours time limit, while ARP was able to find an effective policy in 50\% of the runs, and the effectiveness at higher $\alpha$ increased at a higher action rate. 
\begin{figure}[tb]
  \centering
  \includegraphics[width=0.5\textwidth]{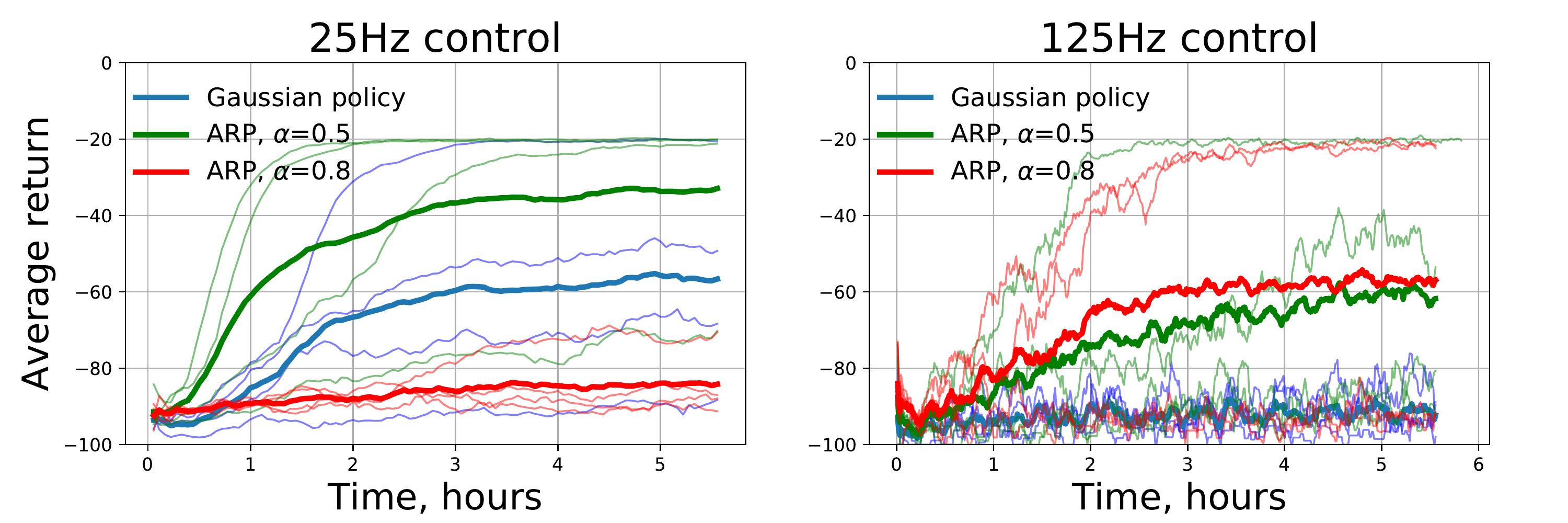}
  \caption{Learning curves on a UR5 Reacher 2D environment with sparse reward at 25Hz and 125Hz velocity control. Thick lines represent average across runs, thin lines show individual curves.}
  \label{fig:ur5}
\end{figure}

\section{Conclusions}
We introduced autoregressive Gaussian policies (ARPs) for temporally coherent exploration in continuous control deep reinforcement learning. The policy form is grounded in the theory of stationary autoregressive stochastic processes. We derived a family of stationary Gaussian autoregressive stochastic processes for an arbitrary order $p$ with continuously adjustable degree of temporal coherence between subsequent observations. We derived an agent policy that implements these processes with a standard agent-environment interface. Empirically we showed that ARPs result in a superior exploration and learning in sparse reward tasks and perform on par or better compared to standard Gaussian policies in dense reward tasks. On physical hardware, ARPs result in smooth trajectories that are safer to execute compared to the trajectories provided by conventional Gaussian exploration.

\bibliographystyle{named}
\DeclareRobustCommand{\VAN}[3]{#3}
\bibliography{correlated_noise_refs}

\onecolumn
\clearpage
\appendix
\section{Proof of Proposition \ref{theorem}}
\label{app:proof}
\begin{lemma}
\label{lemma1}
For any $p \in \mathbb{N}$ and for any $\alpha_k \in [0, 1), k = 1, \dots, p$ the autoregressive process
\begin{equation}
\label{arp2}
\begin{aligned}
X_t &= \sum_{k = 1}^p \tilde{\phi}_k X_{t-k} + Z_t\\
\tilde{\phi}_k &= (-1)^{k + 1}\sum\limits_{1 
\le i_1, i_2, \dots, i_k \le p} \alpha_{i_1} \alpha_{i_2} \dots \alpha_{i_k},\ k = 1, \ldots, p\\
Z_t &\sim \text{WN}(0, \sigma_Z^2),\ \sigma_Z^2 < \infty,
\end{aligned}
\end{equation} 
is stationary.
\end{lemma} 
\begin{proof}
Consider a polynomial
$$P(z) = (z - \alpha_1)(z - \alpha_2)\dots(z - \alpha_p).$$
It can be written in a form
\begin{equation}
\label{pz2}
\begin{aligned}
P(z) &= z^p - \sum_{k = 1}^p\tilde{\phi}_k z^{p-k},\\
\tilde{\phi}_k &= (-1)^{k + 1}\sum\limits_{1 
\le i_1, i_2, \dots, i_k \le p} \alpha_{i_1} \alpha_{i_2} \dots \alpha_{i_k},\ k = 1, \ldots, p 
\end{aligned}
\end{equation} 
consequently $P(z)$ is a characteristic polynomial of the process (\ref{arp2}). By design, $P(z)$ has roots $\alpha_k, k = 1, \dots, p$ which all lie within a unit circle, therefore the process (\ref{arp2}) is stationary.
\end{proof}

\begin{lemma}
\label{lemma2}
Let $\{X_t\}$ be an autoregressive process defined in (\ref{arp2}).
If its white noise component $Z_t$ is Gaussian, i.e. $Z_t \sim \mathcal{N}(0, \sigma_Z^2)$, then $X_t$ are identically distributed normal random variables with zero mean and finite variance $var(X_t)~=~\sigma_X ^2~<~\infty \ \forall t$.
\end{lemma}
\begin{proof}
According to lemma \ref{lemma1} the process (\ref{arp2}) is stationary, meaning $X_t$ are identically distributed random variables with finite variance $\text{var}(X_t) = \sigma_X ^2 < \infty \ \forall t$. If $Z_t$ is Gaussian, then $X_t$ are identically distributed normal variables. Let us denote $\mu_X$ the mean of this distribution. Taking expectation of both sides of (\ref{arp2}) gives

\begin{align*}
\mathbb{E}[X_t] &= \sum_{k = 1}^p \tilde{\phi}_k \mathbb{E}[X_{t-k}] + \mathbb{E}[Z_t] \ \Rightarrow\\
\mu_X &= \sum_{k = 1}^p \tilde{\phi}_k \mu_X \ \Rightarrow\\
\mu_X& (1 - \sum_{k = 1}^p \tilde{\phi}_k) = 0.
\end{align*}
$(1 - \sum_{k = 1}^p \tilde{\phi}_k) \ne 0$ since polynomial (\ref{pz2}) does not have a root $z = 1$, therefore $\mu_X = 0$.
\end{proof}

We established that under Gaussian white noise $Z_t$ the process (\ref{arp2}) represents a series of identically distributed normal variables with zero mean and finite variance. From the linear form of the process it is clear that the variance $\sigma_X^2$ linearly depends on the variance $\sigma_Z^2$ of a white noise component. Scaling $\sigma_Z^2$ by a factor $\beta > 0$ results in scaling $\sigma_X^2$ by the same factor $\beta$. Therefore, it should be possible to pick $\sigma_Z$ such that the variance $\sigma_X^2$ would have any desired positive value, in particular, value of 1. The following lemma formalizes this observation.

\begin{lemma}
\label{lemma3}
For any $p \in \mathbb{N}$ let $\{\phi_k, k = 0, \dots, p\}$ be a set of coefficients corresponding to a stationary AR-p process $\{X_t\}$ defined in (\ref{arp}). Then the linear system 

\begin{equation}
\label{ywst}
\begin{aligned}
& \begin{bmatrix}
\gamma_1 \\
\gamma_2 \\
\gamma_3 \\
\vdots \\
\gamma_p
\end{bmatrix}  = 
\begin{bmatrix}
1 & \gamma_1 & \gamma_2 & \ldots & \gamma_{p-1} \\
\gamma_1 & 1 & \gamma_1 & \ldots & \gamma_{p-2} \\
\gamma_2 & \gamma_1 & 1 & \ldots & \gamma_{p-3} \\
\vdots & \vdots & \vdots & \ddots & \vdots \\
\gamma_{p-1} & \gamma_{p-2} & \gamma_{p-3} & \ldots & 1 \\
\end{bmatrix} 
\begin{bmatrix}
\phi_1 \\
\phi_2 \\
\phi_3 \\
\vdots \\
\phi_p
\end{bmatrix}\\
\text{and}& \\
& 1 = \sum_{i = 1}^p \phi_i \gamma_i + \sigma_Z^2, \\
\end{aligned}
\end{equation}
has a unique solution $(\tilde{\gamma}_1, \ldots, \tilde{\gamma}_p,\tilde{\sigma}_Z^2)$, where $\tilde{\sigma}_Z^2 > 0$. Furthermore, the autoregressive process

\begin{equation}
\begin{aligned}
\tilde{X}_t &= \sum_{k = 1}^p \phi_k \tilde{X}_{t-k} + \tilde{Z}_t,\\
\text{var}(\tilde{Z}_t)&=\tilde{\sigma}_Z^2 \\
\end{aligned}
\end{equation}
is stationary with variance $\tilde{\sigma}_X^2 = \text{var}(\tilde{X}_t) = 1$. 

\end{lemma}
\begin{proof}
For any $\tilde{\sigma}_Z^2 > 0$ stationarity of $\{\tilde{X}_t\}$ follows from stationarity of $\{X_t\}$, since both processes share the same coefficients $\{\phi_k\}$, and therefore, the same characteristic polynomial. 

Since $\{X_t\}$ is stationary, the corresponding system of Yule-Walker equations (\ref{yw}) has a unique solution with respect to $(\gamma_0, \gamma_1, \ldots, \gamma_p)$ \cite[Section 3.1]{brockwell2002introduction}. Notice, however, that the system (\ref{yw}) is homogenous with respect to the variables $(\gamma_0, \gamma_1, \ldots, \gamma_p, \sigma_Z^2)$, meaning that if $(\gamma_0^{\prime}, \gamma_1^{\prime}, \ldots, \gamma_p^{\prime}, \sigma_Z^{\prime 2})$ is a solution, then $(\beta \gamma_0^{\prime}, \beta \gamma_1^{\prime}, \ldots, \beta \gamma_p^{\prime}, \beta \sigma_Z^{\prime2})$ is also a solution $\forall \beta \in \mathcal{R}$.
Therefore, since $\gamma_0 = \text{var}(X_t) > 0$, exists a unique solution $(\tilde{\gamma}_0, \tilde{\gamma}_1, \ldots, \tilde{\gamma}_p, \tilde{\sigma_Z}^2)$ such that $\tilde{\gamma}_0 = 1$. We can find it by substituting $\gamma_0 = 1$ into a linear system (\ref{yw}), resulting in (\ref{ywst}). This solution corresponds to a stationary process $\{\tilde{X}_t\}$ with $\text{var}(\tilde{X}_t) = \tilde{\gamma}_0 = 1 \ \forall t$.

\end{proof}
Now the proof of proposition \ref{theorem} is straightforward. By lemma \ref{lemma1} coefficients $\{\tilde{\phi}_k\}$ correspond to a stationary autoregressive process, therefore by lemma \ref{lemma3} the system (\ref{ywst}) has a unique solution $\tilde{\sigma}_Z^2 > 0$ and the process (\ref{arpn}) is a stationary process with unit variance. Since $Z_t$ is normal, by lemma (\ref{lemma2}) the process (\ref{arpn}) is Gaussian with zero mean. Therefore, $X_t \sim \mathcal{N}(0, 1) \ \forall t$.

\section{Example of an AR-3 Gaussian process}
\label{app:ar3}
A third order process AR-3 defined by (\ref{arpn}) has a form:
\begin{align*}
X_t &= (\alpha_1 + \alpha_2 + \alpha_3)X_{t-1}  - (\alpha_1 \alpha_2 + \alpha_2\alpha_3 + \alpha_1\alpha_3)X_{t-2} + \alpha_1 \alpha_2 \alpha_3 X_{t-3} + \sigma_Z \mathcal{N}(0, 1),
\end{align*}

where $\sigma_Z^2$ is a solution of a system:

\begin{align*}
& \begin{bmatrix}
\gamma_1 \\
\gamma_2 \\
\gamma_3 \\
\end{bmatrix}  = 
\begin{bmatrix}
1 & \gamma_1 & \gamma_2 \\
\gamma_1 & 1 & \gamma_1 \\
\gamma_2 & \gamma_1 & 1 \\
\end{bmatrix} 
\begin{bmatrix}
\phi_1 \\
\phi_2 \\
\phi_3 \\
\end{bmatrix},\\
&\phi_1 = \alpha_1 + \alpha_2 + \alpha_3\\
&\phi_2 =  - (\alpha_1 \alpha_2 + \alpha_2\alpha_3 + \alpha_1\alpha_3)\\
&\phi_3 = \alpha_1 \alpha_2 \alpha_3\\
\text{and}& \\
& 1 = \sum_{i = 1}^3 \phi_i \gamma_i + \sigma_Z^2,
\end{align*}

resulting in
$$
\sigma_Z^2 = \frac{(1 - \alpha_1^2)(1 - \alpha_2^2)(1 - \alpha_3^2)(1 - \alpha_1 \alpha_2)(1 - \alpha_2\alpha_3)(1 - \alpha_1\alpha_3)}{(1 + \alpha_1 \alpha_2 + \alpha_2\alpha_3 + \alpha_1 \alpha_3 - \alpha_1  \alpha_2  \alpha_3  (\alpha_1\alpha_2\alpha_3 + \alpha_1 + \alpha_2 + \alpha_3))}.$$

For any $\alpha_1, \alpha_2, \alpha_3 \in [0, 1)$ this process is stationary with $X_t \sim \mathcal{N}(0, 1)\ \forall t$.

If $\alpha_1 = \alpha_2 = \alpha_3 = \alpha$, the process reduces to 

\begin{align*}
X_t &= 3 \alpha X_{t-1}  - 3 \alpha^2 X_{t-2} + \alpha^3 X_{t-3} + \sigma_Z \mathcal{N}(0, 1),\\
\sigma_Z^2 &= \frac{(1 - \alpha^2)^6}{1 + 3 \alpha^2 - 3 \alpha^4 - \alpha^6}.
\end{align*}

\section{Learning results in simulation with ARPs and OpenAI Baselines TRPO}
\label{app:trpo}
We ran a set of experiments with ARPs ($p$ = 3) and OpenAI Baselines TRPO algorithm. Figures \ref{fig:trpo_mujoco} and \ref{fig:trpo_square} show learning curves in Mujoco and Square environments respectively. The hyper-parameters are specified in the Appendix \ref{app:params}. TRPO delivered a similar performance to PPO in Mujoco tasks (with a more stable performance on a Reacher-v2 task), however PPO produced better results on a sparse reward Square environment.
\begin{figure}[tbh]
\centering
\begin{subfigure}{0.9\textwidth}
  \centering
  \includegraphics[width=1\textwidth]{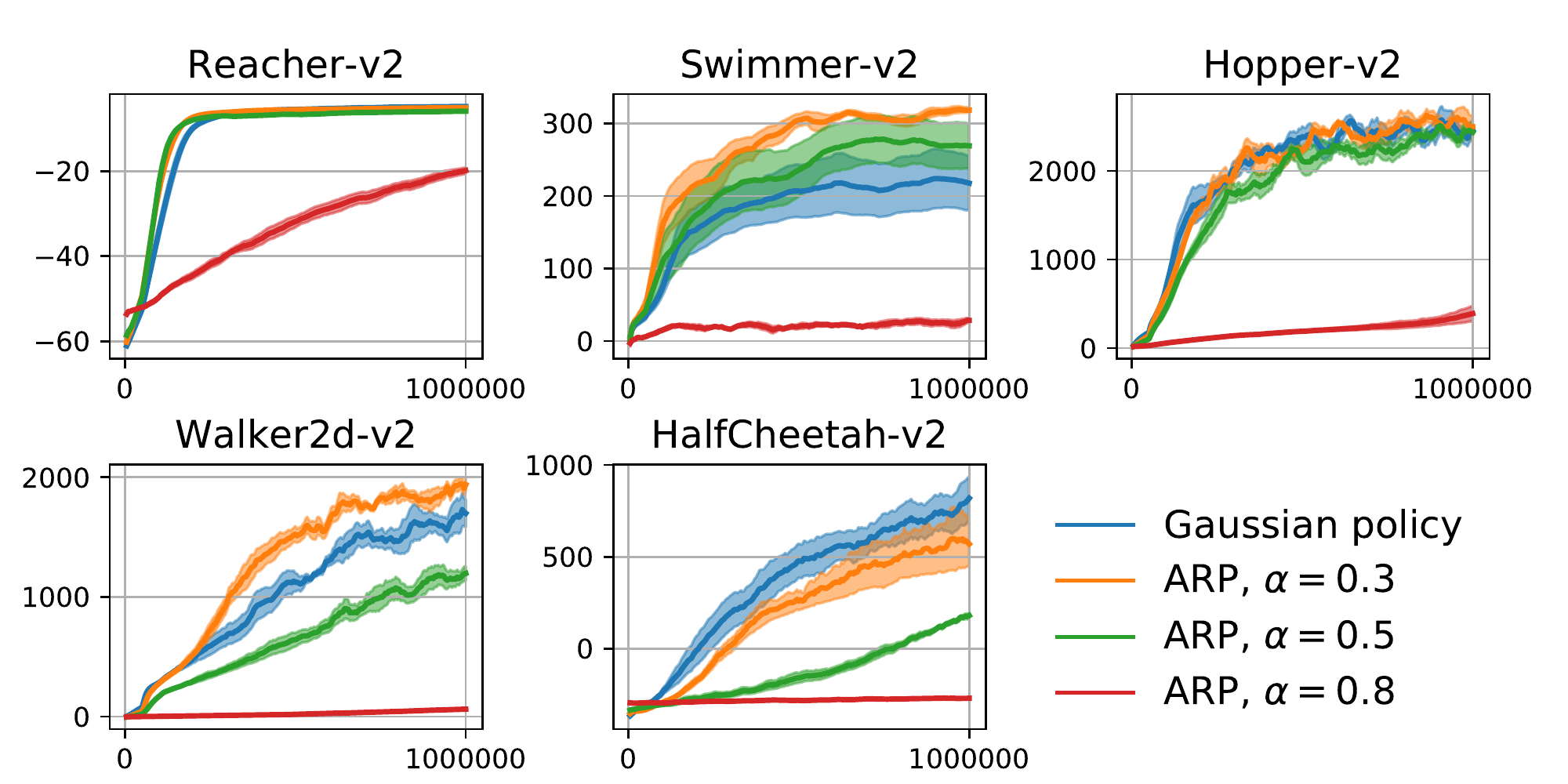}
   \caption{}
  \label{fig:trpo_mujoco}
\end{subfigure}%

\begin{subfigure}{1\textwidth}
  \centering
  \includegraphics[width=1\textwidth]{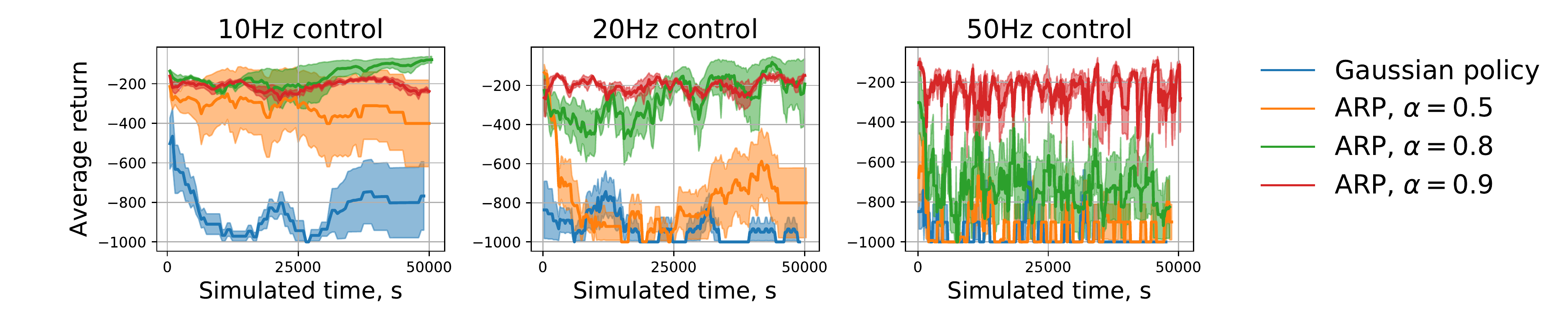}
  \caption{}
  \label{fig:trpo_square}
\end{subfigure}
\caption{Learning curves in (a) Mujoco-based, (b) Square environments with OpenAI Baselines TRPO.}
\end{figure}
\section{Algorithm parameters}
\label{app:params}
In our experiments we used the default Open AI Baselines parameters specified in \url{https://github.com/openai/baselines/blob/master/baselines} with the exception of $\gamma$ and $\lambda$ parameters, for which we used a larger value of 0.995. For experiments in Square environment at 10Hz action rate we used 4 times larger batch and optimization batch sizes to account for longer episodes. For Square and UR5 Reacher 2D environments at higher action rates we used the same parameter values as for basic versions (10Hz and 25Hz respectively), but scaled the batch and the optimization batch sizes accordingly to make sure that the data within a batch at all action rates corresponds to the same amount of simulated time (e.g. for UR5 Reacher 2D at 125Hz we used 5 times bigger batch and optimization batch compared to those in UR5 Reacher 2D at 25Hz). In all experiments we used fully connected networks with the same hidden sizes  to parametrize policy and value networks. For each experiment identical parameters and network architectures were used for standard Gaussian policy and ARP. The table below shows the parameters values for basic versions of the environments: 

\textbf{PPO:}

\begin{center}
\begin{tabular}{ |c|c|c|c| } 
 \hline 
 Hyper-parameter & Square at 10Hz & Mujoco and UR5 Reacher 2D at 25Hz \\ 
 \hline 
 batch size & 8192  & 2048 \\
 \rowcolor{Gray}
 step-size & $4 \times 10^{-3}$ & $4 \times 10^{-3}$ \\ 
 opt. batch size & 256 & 64 \\
  \rowcolor{Gray}
 opt. epochs & 10 & 10 \\ 
  $\gamma $ & 0.995 & 0.995 \\
  \rowcolor{Gray}
$\lambda $ & 0.995 & 0.995 \\
 clip. $\varepsilon$ & 0.2 & 0.2 \\
 \rowcolor{Gray}
 hidden layers & 2 & 2 \\
 hidden sizes & 64 & 64 \\
 \hline
\end{tabular}
\end{center}

\textbf{TRPO:}
\begin{center}
\begin{tabular}{ |c|c|c|c| } 
 \hline 
 Hyper-parameter & Square at 10Hz & Mujoco\\ 
 \hline 
 batch size & 8192  & 1024\\
 \rowcolor{Gray}
 max-kl & 0.01 & 0.01\\ 
 cg-iters & 10 & 10\\
  \rowcolor{Gray}
 vf-iters & 5 & 5\\ 
 vf-step-size & $10^{-3}$ & $10^{-3}$\\
 \rowcolor{Gray}
  $\gamma $ & 0.995 & 0.995\\
 
$\lambda $ & 0.995 & 0.995\\ 
 \rowcolor{Gray}
 hidden layers & 2 & 2\\
 hidden sizes & 64 & 64\\
 \hline
\end{tabular}
\end{center}

\section{Implementation details}
\label{app:details}
The equation (\ref{policy}) in the main text, replicated for convenience below, defines a stationary autoregressive policy distribution in $\tilde{M}^p$ under given parameters $\theta$:
\begin{align*}
\pi_{\theta}(a_t|\tilde{s}_t) &= \mathcal{N}(\mu_{\theta}(s_t) + \sigma_{\theta}(s_t) f_{\theta}(\tilde{s}_t),\ \sigma^2_{\theta}(s_t)\tilde{\sigma}^2_Z I),\\
f_{\theta}(\tilde{s}_t) &= \sum_{k = 1}^p \tilde{\phi}_k \frac{a_{t-k} - \mu_{\theta}(s_{t-k})}{\sigma_{\theta}(s_{t-k})},\\
\tilde{\phi}_k,\ \tilde{\sigma}^2_Z &\text{ are defined by (\ref{arpn2})}.
\end{align*}
The two aspects that need to be considered separately for a practical application are the initialization at the start of a roll-out and the learning updates. 

At $t < p$, we do not include in the model the terms $\tilde{\phi}_k \frac{a_{t-k} - \mu_{\theta}(s_{t-k})}{\sigma_{\theta}(s_{t-k})}$ where $t - k < 0$. This corresponds to having $X_{t-k} = 0,\ t-k < 0$ in the underlying AR process. Initializing an AR process with zero values results in a well behaved time-series that quickly equilibrates to the stationary distribution without large spikes in values. In contrast, initializing from arbitrary values often results in temporal spikes of large values before the process equilibrates to its stationary behaviour. This restriction can be expressed in the following policy formulation:   
\begin{align*}
\pi_{\theta}(a_t|\tilde{s}_t) &= \mathcal{N}(\mu_{\theta}(s_t) + \sigma_{\theta}(s_t) f_{\theta}(\tilde{s}_t),\ \sigma^2_{\theta}(s_t)\tilde{\sigma}^2_Z I),\\
f_{\theta}(\tilde{s}_t) &= \sum_{k = 1}^{\min(p, t)} \tilde{\phi}_k \frac{a_{t-k} - \mu_{\theta}(s_{t-k})}{\sigma_{\theta}(s_{t-k})},\\
\tilde{\phi}_k,\ \tilde{\sigma}^2_Z &\text{ are defined by (\ref{arpn2}).}
\end{align*}
Formally, dependence on $t$ makes the policy non-stationary, however the term $t$ plays a role only within the first $p$ steps in each sample path, and, as long as the exact action $a_0$ is not encountered again in the sample path, which is the case for continuous action spaces,  the policy could be equivalently formulated by having the agent count number of occurences of $a_0$ in each state $\tilde{s}_t$ and based on that decide on the number of terms in $f_{\theta}(\tilde{s}_t)$. Such formulation would make the policy stationary, however it would not allow a similarly compact representation and would unnecessarily burden the formulation. Notice, that proposed formulation also means that $a_0$ will never be part of policy computation, and therefore will not affect future transitions.

As we pointed out in the main text, if the learning updates are performed within an episode, the change in the parameters $\theta$ can temporally distort stationarity of the underlying AR process, since the terms $\tilde{\phi}_k \frac{a_{t-k} - \mu(s_{t-k}, \theta_{\text{new}})}{\sigma(s_{t-k}, \theta_{\text{new}})}$ change their values compared to $\tilde{\phi}_k \frac{a_{t-k} - \mu(s_{t-k}, \theta_{\text{old}})}{\sigma(s_{t-k}, \theta_{\text{old}})}$. This may cause temporal spikes in values of $f_{\theta}(\tilde{s}_t)$ until the process equilibrates again. To avoid this, corresponding previously computed values of terms $\tilde{\phi}_k \frac{a_{t-k} - \mu(s_{t-k}, \theta_{\text{old}})}{\sigma(s_{t-k}, \theta_{\text{old}})}$ where $t - k < t_{\text{update}}$ should be used in computing $f_{\theta}$ for the next $p$ steps after an update.

\end{document}